\documentclass{article}

\def\COMPLETE{}
\usepackage[boxruled]{algorithm2e}
\usepackage{amsmath,amssymb,amstext,amsthm, tikz}

\usepackage[utf8]{inputenc} 
\usepackage[T1]{fontenc}    
\usepackage{hyperref}       
\usepackage{url}            
\usepackage{booktabs}       
\usepackage{amsfonts}       
\usepackage{nicefrac}       
\usepackage{microtype}      

\usepackage{graphicx}
\usepackage{hyperref}

\usepackage{color}
\usepackage[toc,page]{appendix}

\usepackage{xr}

\usepackage{xspace}
\usepackage[inline]{enumitem}
\usepackage{times}
\usepackage{float}
\usepackage{capt-of}

\usepackage{bbm}

\usepackage{tikz}
\usetikzlibrary{shapes, calc, arrows, through, intersections, decorations.pathreplacing, patterns}

\newcommand{\mc}{\mathcal}
\newcommand{\mb}{\mathbf}
\DeclareMathOperator*{\argmin}{arg\,min}
\DeclareMathOperator*{\argmax}{arg\,max}

\newtheorem{theorem}{Theorem}
\newtheorem{lemma}[theorem]{Lemma}
\newtheorem{definition}[theorem]{Definition}
\newtheorem{proposition}[theorem]{Proposition}
\newtheorem{corollary}[theorem]{Corollary}

\title{Clustering with Same-Cluster Queries}

\author{\normalsize
{\bf Hassan Ashtiani} {\textnormal {,}} {\bf Shrinu Kushagra} {\textnormal {and}} {\bf Shai Ben-David} \\
\normalsize David R. Cheriton School of Computer Science \\
\normalsize University of Waterloo,\\
\normalsize Waterloo, Ontario, Canada\\
\normalsize \{mhzokaei,skushagr,shai\}@uwaterloo.ca \\
}

\date{}

\begin{document}
\maketitle

\begin{abstract}
We propose a framework for Semi-Supervised Active Clustering framework (SSAC), where the learner is allowed to interact with a domain expert, asking whether two given instances belong to the same cluster or not. We study the query and computational complexity of clustering in this framework. We consider a setting where the expert conforms to a center-based clustering with a notion of margin.  We show that there is a trade off between computational complexity and query complexity; We prove that for the case of $k$-means clustering (i.e., when the expert conforms to a solution of $k$-means), having access to relatively few such queries allows efficient solutions to otherwise NP hard problems.

In particular, we provide a probabilistic polynomial-time (BPP) algorithm  for clustering in this setting that asks $O\big(k^2\log k + k\log n)$ same-cluster queries and runs with time complexity $O\big(kn\log n)$ (where $k$ is the number of clusters and $n$ is the number of instances). The algorithm succeeds with high probability for data satisfying margin conditions under which, without queries, we show that the problem is NP hard. We also prove a lower bound on the number of queries needed to have a computationally efficient clustering algorithm in this setting.
\end{abstract}

\section{Introduction}

Clustering is a challenging task particularly due to two impediments. The first problem is that clustering, in the absence of domain knowledge, is usually an \emph{under-specified} task; the solution of choice may vary significantly between different intended applications. The second one is that performing clustering under many natural models is computationally hard.

Consider the task of dividing the users of an online shopping service into different groups. The result of this clustering can then be used for example in suggesting similar products to the users in the same group, or for organizing data so that it would be easier to read/analyze the monthly purchase reports. Those different applications may result in conflicting solution requirements. In such cases, one needs to exploit domain knowledge to better define the clustering problem.

Aside from trial and error, a principled way of extracting domain knowledge is to perform clustering using a form of `weak' supervision. For example, Balcan and Blum \cite{balcan2008clustering} propose to use an interactive framework with 'split/merge' queries for clustering. In another work, Ashtiani and Ben-David \cite{ashtiani2015representation} require the domain expert to provide the clustering of a 'small' subset of data.

At the same time, mitigating the computational problem of clustering is critical. Solving most of the common optimization formulations of clustering is NP-hard (in particular, solving the popular $k$-means and $k$-median clustering problems). One approach to address this issues is to exploit the fact that natural data sets usually exhibit some nice properties and likely to avoid the worst-case scenarios. In such cases, optimal solution to clustering may be found efficiently. The quest for notions of niceness that are likely to occur in real data and allow clustering efficiency is still ongoing (see \cite{Ben-David15} for a critical survey of work in that direction).

In this work, we take a new approach to alleviate the computational problem of clustering. In particular, we ask the following question: can weak supervision (in the form of answers to natural queries) help relaxing the computational burden of clustering? This will add up to the other benefit of supervision: making the clustering problem better defined by enabling the accession of domain knowledge through the supervised feedback.

The general setting considered in this work is the following. Let $X$ be a set of elements that should be clustered and $d$ a dissimilarity function over it. The oracle (e.g., a domain expert) has some information about a  target clustering $C^*_X$ in mind. The clustering algorithm has access to $X, d$, and can also make queries about $C^*_X$. The queries are in the form of \emph{same-cluster} queries. Namely, the algorithm can ask whether two elements belong to the same cluster or not. The goal of the algorithm is to find a clustering that meets some predefined clusterability conditions and is consistent with the answers given to its queries. 

We will also consider the case that the oracle conforms with some optimal $k$-means solution. We then show that access to a 'reasonable' number of same-cluster queries can enable us to provide an efficient algorithm for otherwise NP-hard problems.

\subsection{Contributions}
The two main contributions of this paper are the introduction of the semi-supervised active clustering (SSAC) framework and, the rather unusual 
demonstration that access to simple query answers can turn an otherwise NP hard clustering problem into a feasible one. 

Before we explain those results, let us also mention a notion of clusterability (or `input niceness') that we introduce. We define a novel notion of niceness of data, called $\gamma$-margin property that is related to the previously introduced notion of center proximity \cite{awasthi2012center}. The larger the value of $\gamma$, the stronger the assumption becomes, which means that clustering becomes easier. With respect to that $\gamma$ parameter, we get a sharp `phase transition' between $k$-means being NP hard and being optimally solvable in polynomial time\footnote{The exact value of such a threshold $\gamma$ depends on some finer details of the clustering task; whether $d$ is required to be Euclidean and whether the cluster centers must be members of $X$.}.

We focus on the effect of using queries on the computational complexity of clustering. We provide a probabilistic polynomial time (BPP) algorithm for clustering with queries, that succeeds under the assumption that the input satisfies the $\gamma$-margin condition for $\gamma > 1$. This algorithm makes $O\big(k^2\log k + k\log n)$ same-cluster queries to the oracle and runs in $O\big(kn\log n)$ time, where $k$ is the number of clusters and $n$ is the size of the instance set.

On the other hand, we show that without access to query answers, $k$-means clustering is NP-hard even when the solution satisfies $\gamma$-margin property for $\gamma=\sqrt{3.4} \approx 1.84$ and $k=\Theta(n^\epsilon)$ (for any $\epsilon\in (0,1)$). We further show that access to $\Omega(\log k + \log n)$ queries is needed to overcome the NP hardness in that case.
These results, put together, show an interesting phenomenon. Assume that the oracle conforms to an optimal solution of $k$-means clustering and that it satisfies the $\gamma$-margin property for some $1<\gamma \leq \sqrt{3.4}$. In this case, our lower bound means that without making queries $k$-means clustering is NP-hard, while the positive result shows that with a reasonable number of queries the problem becomes efficiently solvable.  

This indicates an interesting (and as far as we are aware, novel) trade-off between query complexity and computational complexity in the clustering domain.

\subsection{Related Work}
This work combines two themes in clustering research; clustering with partial supervision
(in particular, supervision in the form of answers to queries) and the computational complexity of clustering tasks.

Supervision in clustering (sometimes also referred to as `semi-supervised clustering') has been addressed before, mostly in application-oriented works \cite{basu2002semi,basu2004probabilistic, kulis2009semi}. The most common method to convey such supervision is through a set of pairwise \emph{link/do-not-link} constraints on the instances. Note that in contrast to the supervision we address here, in the setting of the papers cited above, the supervision is non-interactive. On the theory side, Balcan et. al \cite{balcan2008clustering} propose a framework for interactive clustering with the help of a user (i.e., an oracle). The queries considered in that framework are different from ours. In particular, the oracle is provided with the current clustering, and tells the algorithm to either split a cluster or merge two clusters. Note that in that setting, the oracle should be able to evaluate the whole given clustering for each query.

Another example of the use of supervision in clustering was provided by Ashtiani and Ben-David \cite{ashtiani2015representation}. They assumed that the target clustering can be approximated by first mapping the data points into a new space and then performing $k$-means clustering. The supervision is in the form of a clustering of a small subset of data (the subset provided by the learning algorithm) and is used to search for such a mapping.

Our proposed setup combines the user-friendliness of \emph{link/don't-link} queries (as opposed to asking the domain expert to answer queries about whole data set clustering, or to cluster sets of data) with the advantages of interactiveness. 

The computational complexity of clustering has been extensively studied. Many of these results are negative, showing that clustering is computationally hard. For example, $k$-means clustering is NP-hard even for $k=2$ \cite{dasgupta2008hardness}, or in a 2-dimensional plane \cite{vattani2009hardness,mahajan2009planar}. In order to tackle the problem of computational complexity, some notions of niceness of data under which the clustering becomes easy have been considered (see \cite{Ben-David15} for a survey).

The closest proposal to this work is the notion of $\alpha$-center proximity introduced by Awasthi et. al \cite{awasthi2012center}. We discuss the relationship of that notion to our notion of margin in Appendix \ref{appendix:gammaMrginVsAlphaCenter}. In the restricted scenario (i.e., when the centers of clusters are selected from the data set), their algorithm efficiently recovers the target clustering (outputs a tree such that the target is a pruning of the tree) for $\alpha > 3$.  Balcan and Liang \cite{balcan2012clustering} improve the assumption to $\alpha > \sqrt{2} + 1$. Ben-David and Reyzin \cite{ben2014data} show that this problem is NP-Hard for $\alpha < 2$. 

Variants of these proofs for our $\gamma$-margin condition
yield the feasibility of $k$-means clustering when the input satisfies the condition with $\gamma >2$ and NP hardness when $\gamma <2$, both in the case of arbitrary (not necessarily Euclidean) metrics\footnote{In particular, the hardness result of \cite{ben2014data} relies on the ability to construct non-Euclidean distance functions. Later in this paper, we prove hardness for $\gamma \leq \sqrt{3.4}$ for Euclidean instances.} .

\section{Problem Formulation}

\subsection{Center-based clustering}

The framework of clustering with queries can be applied to any type of clustering. However, in this work, we focus on a certain family of common clusterings -- center-based clustering in Euclidean spaces\footnote{In fact, our results are all independent of the Euclidean dimension and apply to any Hilbert space.}.

Let ${\mc X}$ be a subset of some Euclidean space, $\mathbb{R}^d$. Let $\mc C_{\mc X} = \{C_1, \ldots, C_k\}$ be a clustering (i.e., a partitioning) of $\mc X$. We say $x_1 \overset{C_{\mc X}}{\sim} x_2$ if $x_1$ and $x_2$ belong to the same cluster according to $C_{\mc X}$. We further denote by $n$ the number of instances ($|{\mc X}|$) and by $k$ the number of clusters.

We say that a clustering $C_{\mc X}$ is \emph{center-based} if there exists a set of centers $\mc \mu = \{\mu_1, \ldots, \mu_k\} \subset \mc R^n$ such that the clustering corresponds to the Voroni diagram over those center points. Namely, for every $x$ in $\mc X$ and $i \leq k$,  $x\in C_i \Leftrightarrow i=\argmin_j d(x,\mu_j)$. 

Finally, we assume that the centers $\mu^*$ corresponding to $C^*$ are the centers of mass of the corresponding clusters. In other words, $\mu^*_i=\frac{1}{|C_i|}\sum_{x\in C^*_i} x$. Note that this is the case for example when the oracle's clustering is the optimal solution to the Euclidean k-means clustering problem.

\subsection{The $\gamma$-margin property}
Next, we introduce a notion of clusterability of a data set, also referred to as `data niceness property'.

\begin{definition}[$\gamma$-margin]
\label{defn:alphacp}
Let $\mc X$ be set of points in metric space $M$. Let $\mc C_{\mc X} = \{C_1, \ldots, C_k\}$ be a center-based clustering of $\mc X$ induced by centers $\mu_1, \ldots, \mu_k \in M$. We say that $\mc C_{\mc X}$ satisfies the $\gamma$-margin property if the following holds. For all $i \in [k]$ and every $x \in C_i$ and $y \in \mc X \setminus C_i$,


$$\gamma d(x, \mu_i) < d(y, \mu_i)$$
\end{definition}

Similar notions have been considered before in the clustering literature. The closest one to our $\gamma$-margin is the notion of $\alpha$-center proximity \cite{balcan2012clustering,awasthi2012center}. We discuss the relationship between these two notions in appendix \ref{appendix:gammaMrginVsAlphaCenter}.

\subsection{The algorithmic setup}
For a clustering $C^*=\{ C^*_1, \ldots C^*_k\}$, a $C^*$-oracle is a function ${\mc O}_{C^*}$ that answers queries according to that clustering. One can think of such an oracle as a user that has some idea about its desired clustering, enough to answer the algorithm's queries. The clustering algorithm then tries to recover $C^*$ by querying a $C^*$-oracle. The following notion of query is arguably most intuitive.

\begin{definition}[Same-cluster Query]
A same-cluster query asks whether two instances $x_1$ and $x_2$ belong to the same cluster, i.e., 
$${\mc O}_{C^*}(x_1, x_2) = \left\{
	\begin{array}{ll}
		\mbox{true }  & \mbox{if } x_1 \overset{C^*}{\sim} x_2   \\
		\mbox{false } & o.w. 
	\end{array}
\right. $$

(we omit the subscript $C^*$ when it is clear from the context).
\end{definition}

\begin{definition}[Query Complexity]
\label{definition:QueryComplexity}
An SSAC instance is determined by the tuple $(\mc X, d, C^*)$. 
We will consider families of such instances determined by niceness conditions on their oracle clusterings $C^*$.
\begin{enumerate}
\item A SSAC algorithm $\mc A$ is called a $q$-solver for a family $G$ of such instances, if for every instance $(\mc X, d, C^*) \in G$, it can recover $C^*$ by having access to $(\mc X, d)$ and making at most $q$ queries to a $C^*$-oracle. 

\item Such an algorithm is a polynomial $q$-solver if its time-complexity is polynomial in $|\mc X|$ and $|C^*|$ (the number of clusters).

\item  We say $G$ admits an $O(q)$ query complexity if there exists an algorithm $\mc A$ that is a polynomial $q$-solver for every clustering instance in $G$.
\end{enumerate}

\end{definition}

\section{An Efficient SSAC Algorithm}
\label{section:clusteringWithQuery}

In this section we provide an efficient algorithm for clustering with queries. The setting is the one described in the previous section. In particular, it is assumed that the oracle has a center-based clustering in his mind which satisfies the $\gamma$-margin property. The space is Euclidean and the center of each cluster is the center of mass of the instances in that cluster. The algorithm not only makes same-cluster queries, but also another type of query defined as below.

\begin{definition}[Cluster-assignment Query]
A cluster-assignment query asks the cluster index that an instance $x$ belongs to. In other words ${\mc O_{C^*}}(x) = i$ if and only if $x \in C^*_i$.
\end{definition}

Note however that each cluster-assignment query can be replaced with $k$ same-cluster queries (see appendix \ref{appendix:diffQueryModels} in supplementary material). Therefore, we can express everything in terms of the more natural notion of same-cluster queries, and the use of cluster-assignment query is just to make the representation of the algorithm simpler.

Intuitively, our proposed algorithm does the following. In the first phase, it tries to approximate the center of one of the clusters. It does this by asking cluster-assignment queries about a set of randomly (uniformly) selected point, until it has a sufficient number of points from at least one cluster (say $C_p$). It uses the mean of these points, $\mu_p^\prime$, to approximate the cluster center. 

In the second phase, the algorithm recovers all of the instances belonging to $C_p$. In order to do that, it first sorts all of the instances based on their distance to $\mu_p^\prime$. By showing that all of the points in $C_p$ lie inside a sphere centered at $\mu_p^\prime$ (which does not include points from any other cluster), it tries to find the radius of this sphere by doing binary search using same-cluster queries. After that, the elements in $C_p$ will be located and can be removed from the data set. The algorithm repeats this process $k$ times to recover all of the clusters.



The details of our approach is stated precisely in Algorithm \ref{alg:steinerQueryPositive}. Note that $\beta$ is a small constant\footnote{It corresponds to the constant appeared in generalized Hoeffding inequality bound, discussed in Theorem \ref{thm:genHoeff} in appendix \ref{appendixsection:conIneq} in supplementary materials.}. Theorem \ref{thm:steinerQueryPositive} shows that if $\gamma > 1$ then our algorithm recovers the target clustering with high probability. Next, we give bounds on the time and query complexity of our algorithm. Theorem \ref{thm:steinerQueryPositiveComplexity} shows that our approach needs $O(k\log n + k^2\log k)$ queries and runs with time complexity $O(kn\log n)$.

\RestyleAlgo{boxruled} 
\SetAlgoNoLine
\begin{algorithm}[h]
 \KwIn{Clustering instance $\mc X$, oracle $\mc O$, the number of clusters $k$ and parameter $\delta \in (0, 1)$}
 \KwOut{A clustering $\mc C$ of the set $\mc X$}

 \vspace{0.5em} $\mc C = \{\}$,
 $\mc S_{1} = \mc X$,
 $\eta = \beta \frac{\log k + \log(1/\delta)}{(\gamma-1)^4}$\\
 \For{$i = 1$ to $k$}{
 	\vspace{0.5em}\textbf{Phase 1}\\
 	$l = k \eta + 1$\;
	$Z \sim U^l[\mc S_i]$    \mbox{       } //   Draws $l$ independent elements from $\mc S_i$ uniformly at random\\
	For $1 \le t \le i$,\\
             \mbox{            } $Z_t = \{x \in Z : {\mc O}(x)= t\}.$ \mbox{    } //Asks cluster-assignment queries about the members of $Z$\\
$p = \argmax_t |Z_t|$\\
	$\mu_p' := \frac{1}{|Z_p|}\sum_{x \in Z_p} x$.\\

	\vspace{0.5em}\textbf{Phase 2}\\
    // We know that there exists $r_i$ such that $\forall x\in {\mc S_i}$, $x\in C_i \Leftrightarrow d(x, \mu^\prime_i)< r_i$.\\ 
    // Therefore, $r_i$ can be found by simple binary search\\
    
	$\widehat{\mc S_i}$ = Sorted$(\{\mc S_i\})$ \mbox{       }// Sorts elements of $\{x: x\in \mc S_i\}$ in increasing order of $d(x, \mu_p')$.\\    
	 $r_i = $ BinarySearch$(\widehat{\mc S_i})$ \mbox{      } //This step takes up to $O(\log|{\mc S_i}|)$ same-cluster queries\\

	$C_p' = \{x \in \mc S_i: d(x, \mu_p') \le r_i\}$.\\
	$S_{i+1} = S_{i}\setminus C_p'$.\\
	$\mc C = \mc C \cup \{C_p'\}$
 }
 \label{alg:steinerQueryPositive}
 \caption{Algorithm for $\gamma(> 1)$-margin instances with queries}
\end{algorithm}

\begin{lemma}
\label{lemma:hasGammaMargin}
Let $(\mc X, d, C)$ be a clustering instance, where $C$ is center-based and satisfies the $\gamma$-margin property. Let $\mu$ be the set of centers corresponding to the centers of mass of $C$. Let $\mu_i'$ be such that $d(\mu_i, \mu_i') \le r(C_i)\epsilon$, where $r(C_i) = \max_{x\in C_i}d(x, \mu_i)$ . Then $\gamma \ge 1 + 2\epsilon$ implies that 

\begin{center}$\forall x \in C_i, \forall y \in {\mc X} \setminus C_i \Rightarrow d(x, \mu_i') < d(y, \mu_i')$\end{center}
\end{lemma}

\begin{proof}
Fix any $x \in C_i$ and $y \in C_j$. $d(x, \mu_i') \le d(x, \mu_i)+d(\mu_i, \mu_i') \le r(C_i) (1+\epsilon)$. Similarly, $d(y, \mu_i') \ge d(y, \mu_i) - d(\mu_i, \mu_i') > (\gamma -\epsilon)r(C_i)$. Combining the two, we get that $d(x, \mu_i') < \frac{1+\epsilon}{\gamma-\epsilon}d(y, \mu_i')$. 
\end{proof}

\begin{lemma}
\label{lemma:phase1}
Let the framework be as in Lemma \ref{lemma:hasGammaMargin}. Let $Z_p, C_p, \mu_p$, $\mu_p^\prime$ and $\eta$ be defined as in Algorhtm \ref{alg:steinerQueryPositive}, and $\epsilon = \frac{\gamma - 1}{2}$. If $|Z_p| > \eta$, then the probability that $d(\mu_p, \mu_p^\prime) > r(C_p)\epsilon$ is at most $\frac{\delta}{k}$.
\end{lemma}
\begin{proof}
Define a uniform distribution $U$ over $C_p$. Then $\mu_p$ and $\mu_p^\prime$ are the true and empirical mean of this distribution. Using a standard concentration inequality (Thm. \ref{thm:genHoeff} from Appendix \ref{appendixsection:conIneq}) shows that the empirical mean is close to the true mean, completing the proof. 

\end{proof}

\begin{theorem}
\label{thm:steinerQueryPositive}
Let $(\mc X, d, C)$ be a clustering instance, where $C$ is center-based and satisfies the $\gamma$-margin property. Let $\mu_i$ be the center of mass of $C_i$.
Assume $\delta \in (0, 1)$ and $\gamma > 1$. Then with probability at least $1-\delta$, Algorithm \ref{alg:steinerQueryPositive} outputs $C$.
\end{theorem}

\begin{proof}
In the first phase of the algorithm we are making $l>k\eta$ cluster-assignment queries. Therefore, using the pigeonhole principle, we know that there exists cluster index $p$ such that $|Z_p| > \eta$. Then Lemma \ref{lemma:phase1} implies that the algorithm chooses a center $\mu_p^\prime$ such that with probability at least $1-\frac{\delta}{k}$ we have $d(\mu_p, \mu_p^\prime) \le r(C_p)\epsilon$. By Lemma \ref{lemma:hasGammaMargin}, this would mean that $d(x, \mu_p^\prime) < d(y, \mu_p')$ for all $x \in C_p$ and $y \not\in C_p$. Hence, the radius $r_i$ found in the phase two of Alg. \ref{alg:steinerQueryPositive} is such that $r_{i} = \max\limits_{x \in C_p} d(x, \mu_p^\prime)$. This implies that $C_p^\prime$ (found in phase two) equals to $C_p$. Hence, with probability at least $1-\frac{\delta}{k}$ one iteration of the algorithm successfully finds all the points in a cluster $C_p$. Using union bound, we get that with probability at least $1-k\frac{\delta}{k} = 1-\delta$, the algorithm recovers the target clustering.
\end{proof}

\begin{theorem}
\label{thm:steinerQueryPositiveComplexity}
Let the framework be as in theorem \ref{thm:steinerQueryPositive}. Then Algorithm \ref{alg:steinerQueryPositive} 
\begin{itemize}[nolistsep,noitemsep]
\item Makes $O\big(k\log n + k^2\frac{\log k + \log (1/\delta)}{(\gamma - 1)^4}\big)$ same-cluster queries to the oracle $\mc O$.
\item Runs in $O\big(kn\log n + k^2\frac{\log k + \log (1/\delta)}{(\gamma - 1)^4}\big)$ time.
\end{itemize}
\end{theorem}

\begin{proof}
In each iteration (i) the first phase of the algorithm takes $O(\eta)$ time and makes $\eta+1$ cluster-assignment queries (ii) the second phase takes $O(n\log n)$ times and makes $O(\log n)$ same-cluster queries. Each cluster-assignment query can be replaced with $k$ same-cluster queries; therefore, each iteration runs in $O(k\eta + n\log n)$ and uses $O(k\eta + \log n)$ same-cluster queries. By replacing $\eta = \beta\frac{\log k + \log(1/\delta)}{(\gamma-1)^4}$ and noting that there are $k$ iterations, the proof will be complete.
\end{proof}

\begin{corollary}
The set of Euclidean clustering instances that satisfy the $\gamma$-margin property for some $\gamma > 1$ admits query complexity $O\big(k\log n + k^2\frac{\log k + \log (1/\delta)}{(\gamma - 1)^4}\big)$. 
\end{corollary}

\section{Hardness Results}
\label{section:lowerBounds}

\subsection{Hardness of Euclidean $k$-means with Margin}

Finding $k$-means solution without the help of an oracle is generally computationally hard. In this section, we will show that solving Euclidean $k$-means remains hard even if we know that the optimal solution satisfies the $\gamma$-margin property for $\gamma=\sqrt{3.4}$. In particular, we show the hardness for the case of $k=\Theta(n^\epsilon)$ for any $\epsilon\in (0,1)$.


In Section \ref{section:clusteringWithQuery}, we proposed a polynomial-time algorithm that could recover the target clustering using $O(k^2\log k +k\log n)$ queries, assuming that the clustering satisfies the $\gamma$-margin property for $\gamma>1$. Now assume that the oracle conforms to the optimal $k$-means clustering solution. In this case, for $1<\gamma\le \sqrt{3.4} \approx 1.84$, solving $k$-means clustering would be NP-hard without queries, while it becomes efficiently solvable with the help of an oracle \footnote{To be precise, note that the algorithm used for clustering with queries is probabilistic, while the lower bound that we provide is for deterministic algorithms. However, this implies a lower bound for randomized algorithms as well unless $BPP\neq P$}. 

Given a set of instances $\mc X \subset \mb R ^d$, the $k$-means clustering problem is to find a clustering $\mc C = \{C_1, \ldots, C_k\}$ which minimizes $f(\mc C) = \sum\limits_{C_i} \min\limits_{\mu_i\in {\mb R}^d}\sum\limits_{x\in C_i} \|x - \mu_i \|_2^2$. The decision version of $k$-means is, given some value $L$, is there a clustering $\mc C$ with cost $\le L$? The following theorem is the main result of this section. 

\begin{theorem}
\label{thm:gammaLower}
Finding the optimal solution to Euclidean $k$-means objective function is NP-hard when $k=\Theta(n^\epsilon)$ for any $\epsilon \in (0,1)$, even when the optimal solution satisfies the $\gamma$-margin property for $\gamma = \sqrt{3.4}$.
\end{theorem}

This results extends the hardness result of \cite{ben2014data} to the case of Euclidean 
metric, rather than arbitrary one, and to the $\gamma$-margin condition (instead of the $\alpha$-center proximity there). The full proof is rather technical and is deferred to the supplementary material (appendix \ref{appendix:lowerBoundProof}). In the next sections, we provide an outline of the proof. 

\subsubsection{Overview of the proof}

Our method to prove Thm. \ref{thm:gammaLower} is based on the approach employed by \cite{vattani2009hardness}. However, the original construction proposed in \cite{vattani2009hardness} does not satisfy the $\gamma$-margin property. Therefore, we have to modify the proof by setting up the parameters of the construction more carefully. 

To prove the theorem, we will provide a reduction from the problem of Exact Cover by 3-Sets (\textsc{X3C}) which is NP-Complete \cite{garey2002computers}, to the decision version of $k$-means.

\begin{definition}[\textsc{X3C}]
Given a set $U$ containing exactly $3m$ elements and a collection $\mc S = \{S_1, \ldots, S_l\}$ of subsets of $U$ such that each $S_i$ contains exactly three elements, does there exist $m$ elements in $\mc S$ such that their union is $U$? 
\end{definition}

We will show how to translate each instance of X3C, $(U,\mc S)$, to an instance of $k$-means clustering in the Euclidean plane, $X$. In particular, $X$ has a grid-like structure consisting of $l$ rows (one for each $S_i$) and roughly $6m$ columns (corresponding to $U$) which are embedded in the Euclidean plane. The special geometry of the embedding makes sure that any low-cost $k$-means clustering of the points (where $k$ is roughly $6ml$) exhibits a certain structure. In particular, any low-cost $k$-means clustering could cluster each row in only two ways; One of these corresponds to $S_i$ being included in the cover, while the other means it should be excluded. We will then show that $U$ has a cover of size $m$ if and only if $X$ has a clustering of cost less than a specific value $L$. Furthermore, our choice of embedding makes sure that the optimal clustering satisfies the $\gamma$-margin property for $\gamma=\sqrt{3.4} \approx 1.84$.

\subsubsection{Reduction design}
Given an instance of X3C, that is the elements $U = \{1, \ldots, 3m\}$ and the collection $\mc S$, we construct a set of points $X$ in the Euclidean plane which we want to cluster. Particularly, $X$ consists of a set of points $H_{l,m}$ in a grid-like manner, and the sets $Z_i$ corresponding to $S_i$. In other words, $X = H_{l,m} \cup (\cup_{i=1}^{l-1} Z_i)$. 

The set $H_{l,m}$ is as described in Fig. \ref{fig:lowerBoundComponent}. The row $R_i$ is composed of $6m + 3$ points $\{s_i, r_{i, 1}, \ldots, r_{i, 6m+1}, f_i\}$. Row $G_i$ is composed of $3m$ points $\{g_{i,1}, \ldots, g_{i, 3m}\}$. The distances between the points are also shown in Fig. \ref{fig:lowerBoundComponent}. Also, all these points have weight $w$, simply meaning that each point is actually a set of $w$ points on the same location.

Each set $Z_i$ is constructed based on $S_i$. In particular, $Z_i = \cup_{j\in [3m]} B_{i,j}$, where $B_{i,j}$ is a subset of $\{x_{i,j},x_{i,j}',y_{i,j},y_{i,j}'\}$ and is constructed as follows: $x_{i,j} \in B_{i,j}$ iff $j \not\in S_i$, and $x_{i,j}' \in B_{i,j}$ iff $j \in S_i$. Similarly,  $y_{i,j} \in B_{i,j}$ iff $j \not\in S_{i+1}$, and $y_{i,j}' \in B_{i,j}$ iff $j \in S_{i+1}$. Furthermore, $x_{i, j}, x_{i,j}^\prime, y_{i,j}$ and $y_{i, j}^\prime$ are specific locations as depicted in Fig. \ref{fig:ZFig}. In other words, exactly one of the locations $x_{i,j}$ and $x_{i,j}^\prime$, and one of $y_{i,j}$ and $y_{i,j}^\prime$ will be occupied. We set the following parameters. 
\vspace{-0.1in}
\begin{align*}
&h = \sqrt{5}, d = \sqrt{6}, \epsilon = \frac{1}{w^2}, \lambda = \frac{2}{\sqrt{3}}h, k = (l-1)3m + l(3m+2)\\
& L_1 = (6m+3)wl, L_2 = 3m(l-1)w, L = L_1 + L_2 - m\alpha, \alpha = \frac{d}{w}-\frac{1}{2w^3}
\end{align*}

  \begin{figure}[!tbp]
  \centering
  \begin{minipage}[b]{0.49\textwidth}
    \resizebox{\linewidth}{!}{\ifdefined\COMPLETE
\else
\documentclass[12pt]{article}
\usepackage{tikz}
\usetikzlibrary{shapes, calc, arrows, through, intersections, decorations.pathreplacing, patterns}

\begin{document}
\fi
\def\alph{$2+\sqrt{5}$}
\begin{tikzpicture}[scale=7]

	\node[label=180:$R_1$] at (0.0,0) {$\diamond$};	
	\node[] at (0.1,0) {$\bullet$};
	\node[] at (0.2,0) {$\bullet$};
	\node[] at (0.3,0) {$\bullet$};
	\node[] at (0.4,0) {$\bullet$};
	\node[] at (0.6,0) {$\ldots$};
	
	\node[] at (0.8,0) {$\bullet$};
	\node[] at (0.9,0) {$\bullet$};
	\node[] at (1.0,0) {$\diamond$};

	\node[label=180:$G_1$] at (0.0,-0.1) {};	
	\node[] at (0.2,-0.1) {$\circ$};
	\node[] at (0.4,-0.1) {$\circ$};
	\node[] at (0.6,-0.1) {$\ldots$};
	
	\node[] at (0.8,-0.1) {$\circ$};
	
	\node[label=180:$R_2$] at (0.0,-0.2) {$\diamond$};	
	\node[] at (0.1,-0.2) {$\bullet$};
	\node[] at (0.2,-0.2) {$\bullet$};
	\node[] at (0.3,-0.2) {$\bullet$};
	\node[] at (0.4,-0.2) {$\bullet$};
	\node[] at (0.6,-0.2) {$\ldots$};
	
	\node[] at (0.8,-0.2) {$\bullet$};
	\node[] at (0.9,-0.2) {$\bullet$};
	\node[] at (1.0,-0.2) {$\diamond$};

	\node[label=180:$G_{l-1}$] at (0.0,-0.4) {};	
	\node[] at (0.2,-0.4) {$\circ$};
	\node[] at (0.4,-0.4) {$\circ$};
	\node[] at (0.6,-0.4) {$\ldots$};
	
	\node[] at (0.8,-0.4) {$\circ$};
	
	\node[label=180:$R_l$] at (0.0,-0.5) {$\diamond$};	
	\node[] at (0.1,-0.5) {$\bullet$};
	\node[] at (0.2,-0.5) {$\bullet$};
	\node[] at (0.3,-0.5) {$\bullet$};
	\node[] at (0.4,-0.5) {$\bullet$};
	\node[] at (0.6,-0.5) {$\ldots$};
	
	\node[] at (0.8,-0.5) {$\bullet$};
	\node[] at (0.9,-0.5) {$\bullet$};
	\node[] at (1.0,-0.5) {$\diamond$};

	\node[label=90:\scriptsize$d$] at (0.05,0.0) {};
	\node[label=90:\scriptsize$2$] at (0.15,0.0) {};
	\node[label=90:\scriptsize$2$] at (0.85,0.0) {};
	\node[label={90:\scriptsize $d-\epsilon$}] at (0.97,0.0) {};
	\node[label=90:\scriptsize$4$] at (0.3,-0.1) {};

	\draw[<->] (0.02, 0.0) -- (0.08, 0.0);
	\draw[<->] (0.12, 0.0) -- (0.18, 0.0);
	\draw[<->] (0.82, 0.0) -- (0.88, 0.0);
	\draw[<->] (0.92, 0.0) -- (0.98, 0.0);

	\draw[<->] (0.38, -0.1) -- (0.22, -0.1);
\end{tikzpicture}

\ifdefined\COMPLETE
\else
\end{document}
\fi}
    \caption{Geometry of $H_{l,m}$. This figure is similar to Fig. 1 in \cite{vattani2009hardness}.  
    Reading from letf to right, each row $R_i$ consists of a diamond ($s_i$), $6m+1$ bullets ($r_{i,1},\ldots,r_{i,6m+1}$), and another diamond ($f_i$). Each rows $G_i$ consists of $3m$ circles ($g_{i,1}, \ldots, g_{i,3m}$).}
    \label{fig:lowerBoundComponent}
  \end{minipage}
  \hfill
  \begin{minipage}[b]{0.49\textwidth}
    \ifdefined\COMPLETE
\else
\documentclass[12pt]{article}
\usepackage{tikz}
\usetikzlibrary{shapes, calc, arrows, through, intersections, decorations.pathreplacing, patterns}

\begin{document}
\fi
\def\alph{$2+\sqrt{5}$}
\begin{tikzpicture}[scale=7]

	\node[label=90:\scriptsize $r_{i,2j-1}$] at (0.1,0) {$\bullet$};
	\node[label=90:\scriptsize $r_{i,2j}$] at (0.5,0) {$\bullet$};
	\node[label=90:\scriptsize $r_{i,2j+1}$] at (0.9,0) {$\bullet$};
	
	\node[label=180:\scriptsize $\sqrt{h^2-1}$] at (0.5,-0.07) {};
	\node[label=180:\scriptsize $x_{i,j}$] at (0.5,-0.15) {$\bullet$};

	\node[label=180:\scriptsize $h$] at (0.83,-0.1) {};
	\node[label=180:\scriptsize $x_{i,j}'$] at (0.7,-0.25) {$\bullet$};
	
	\node[label=270:\scriptsize $g_{i,j}$] at (0.46,-0.4) {$\circ$};

	\node[label=180:\scriptsize $y_{i,j}$] at (0.5,-0.65) {$\bullet$};
	\node[label=180:\scriptsize $y_{i,j}'$] at (0.7,-0.55) {$\bullet$};

	\node[label=270:\scriptsize $r_{i+1,2j-1}$] at (0.1,-0.8) {$\bullet$};
	\node[label=270:\scriptsize $r_{i+1,2j}$] at (0.5,-0.8) {$\bullet$};
	\node[label=270:\scriptsize $r_{i+1,2j+1}$] at (0.9,-0.8) {$\bullet$};

	\node[label=180:\scriptsize $\sqrt{h^2-1}$] at (0.5,-0.07) {};

	\node[label=180:\scriptsize $\alpha$] at (0.4,-0.43) {};
	\node[label=90:\scriptsize $1$] at (0.6,-0.02) {};
	\node[label=90:\scriptsize $2$] at (0.3,-0.02) {};

	\draw[<->] (0.5,-0.13) -- (0.5,-0.02);
	\draw[<->] (0.26,-0.4) -- (0.44,-0.4);
	\draw[dotted,-] (0.7, 0.05) -- (0.7, -0.85);
	\draw[<->] (0.74, 0.0) -- (0.74, -0.25);
	\draw[<->] (0.68, 0.0) -- (0.52, 0.0);
	\draw[<->] (0.12, 0.0) -- (0.48, 0.0);
	\draw[black, thick] (0.5,-0.4) circle (0.25);
\end{tikzpicture}

\ifdefined\COMPLETE
\else
\end{document}
\fi
    \caption{The locations of $x_{i,j}$, $x_{i,j}'$, $y_{i,j}$ and $y_{i,j}'$ in the set $Z_i$. Note that the point $g_{i,j}$ is not vertically aligned with $x_{i, j}$ or $r_{i, 2j}$. This figure is adapted from \cite{vattani2009hardness}.}
    \label{fig:ZFig}
  \end{minipage}
\end{figure}

\begin{lemma}
\label{lemma:kmeansEquivalenceX3C}
The set $X = H_{l,n} \cup Z$ has a $k$-clustering of cost less or equal to $L$ if and only if there is an exact cover for the X3C instance.
\end{lemma}

\begin{lemma}
\label{lemma:gammaLower}
Any $k$-clustering of $X = H_{l,n} \cup Z$ with cost $\le L$ has the $\gamma$-margin property where $\gamma = \sqrt{3.4}$. Furthermore, $k = \Theta(n^{\epsilon})$.
\end{lemma}

The proofs are provided in Appendix \ref{appendix:lowerBoundProof}. Lemmas \ref{lemma:kmeansEquivalenceX3C} and \ref{lemma:gammaLower} together show that $X$ has a $k$-clustering of cost $\le L$ satisfying the $\gamma$-margin property (for $\gamma = \sqrt{3.4}$) if and only if there is an exact cover by $3$-sets for the X3C instance. This completes the proof of our main result (Thm. \ref{thm:gammaLower}). 

\subsection{Lower Bound on the Number of Queries}

In the previous section we showed that $k$-means clustering is NP-hard even under $\gamma$-margin assumption (for $\gamma < \sqrt{3.4} \approx 1.84$). On the other hand, in Section \ref{section:clusteringWithQuery} we showed that this is not the case if the algorithm has access to an oracle. In this section, we show a lower bound on the number of queries needed to provide a polynomial-time algorithm for $k$-means clustering under margin assumption.

\begin{theorem}
\label{thm:queryLower}
For any $\gamma \le \sqrt{3.4}$, finding the optimal solution to the $k$-means objective function is NP-Hard even when the optimal clustering satisfies the $\gamma$-margin property and the algorithm can ask $O(\log k + \log |\mc X|)$ same-cluster queries.
\end{theorem}
\begin{proof}
Proof by contradiction: assume that there is polynomial-time algorithm $\mc A$ that makes $O(\log k + \log |\mc X|)$ same-cluster queries to the oracle. Then, we show there exists another algorithm $\mc A^\prime$ for the same problem that is still polynomial but uses no queries. However, this will be a contradiction to Theorem \ref{thm:gammaLower}, which will prove the result.

In order to prove that such $\mc A^\prime$ exists, we use a `simulation' technique. Note that $\mc A$ makes only $q<\beta(\log k + \log |\mc X|)$ binary queries, where $\beta$ is a constant. The oracle therefore can respond to these queries in maximum $2^{q} < k^\beta|\mc X|^\beta$ different ways. Now the algorithm $\mc A^\prime$ can try to simulate all of $k^\beta|\mc X|^\beta$ possible responses by the oracle and output the solution with minimum $k$-means clustering cost. Therefore, $\mc A^\prime$ runs in polynomial-time and is equivalent to $\mc A$.
\end{proof}

\section{Conclusions and Future Directions}
In this work we introduced a framework for semi-supervised active clustering (SSAC) with same-cluster queries. Those queries can be viewed as a natural way for a clustering mechanism to gain domain knowledge, without which clustering is an under-defined task. The focus of our analysis was the computational and query complexity of 
such SSAC problems, when the input data set satisfies a clusterability condition -- the $\gamma$-margin property.

Our main result shows that access to a limited number of such query answers (logarithmic in the size of the data set and quadratic in the number of clusters) allows efficient successful clustering under conditions (margin parameter between 1 and $\sqrt{3.4} \approx 1.84$) that render the problem NP-hard without the help of such a query mechanism.  
 We also provided a lower bound indicating that at least $\Omega(\log kn)$ queries are needed to make those NP hard problems feasibly solvable.

With practical applications of clustering in mind, a natural extension of our model is to allow the oracle (i.e., the domain expert) to refrain from answering a certain fraction of the queries, or to make a certain number of errors in its answers. It would be interesting to analyze how the performance guarantees of SSAC algorithms behave as a function of such abstentions and error rates. Interestingly, we can modify our algorithm to handle a sub-logarithmic number of abstentions by chekcing all possible orcale answers to them (i.e., similar to the ``simulation'' trick in the proof of Thm. \ref{thm:queryLower}).

\subsubsection*{Acknowledgments}
We would like to thank Samira Samadi and Vinayak Pathak for helpful discussions on the topics of this paper. 

\bibliographystyle{alpha}
\bibliography{activeClustering}

\newpage


\appendix
\section{Relationships Between Query Models}
\label{appendix:diffQueryModels}

\begin{proposition}
Any clustering algorithm that uses only $q$ same-cluster queries can be adjusted to use $2q$ cluster-assignment queries (and no same-cluster queries) with the same order of time complexity.
\end{proposition}
\begin{proof}
We can replace each same-cluster query with two cluster-assignment queries as in $Q(x_1,x_2)={\mathbbm{1}}\{Q(x_1)=Q(x_2))\}$.
\end{proof}

\begin{proposition}
Any algorithm that uses only $q$ cluster-assignment queries can be adjusted to use $kq$ same-cluster queries (and no cluster-assignment queries) with at most a factor $k$ increase in computational complexity, where $k$ is the number of clusters.
\end{proposition}
\begin{proof}
If the clustering algorithm has access to an instance from each of $k$ clusters (say $x_i\in X_i$), then it can simply simulate the cluster-assignment query by making $k$ same-cluster queries ($Q(x) = \argmax_{i}\mathbbm{1}\{Q(x, x_i)\}$). Otherwise, assume that at the time of querying $Q(x)$ it has only instances from $k^\prime<k$ clusters. In this case, the algorithm can do the same with the $k^\prime$ instances and if it does not find the cluster, assign $x$ to a new cluster index. This will work, because in the clustering task the output of the algorithm is a partition of the elements, and therefore the indices of the clusters do not matter.
\end{proof}


\section{Comparison of $\gamma$-Margin and $\alpha$-Center Proximity}
\label{appendix:gammaMrginVsAlphaCenter}

In this paper, we introduced the notion of $\gamma$-margin niceness property. We further showed upper and lower bounds on the computational complexity of clustering under this assumption. It is therefore important to compare this notion with other previously-studied clusterability notions.


An important notion of niceness of data for clustering is $\alpha$-center proximity property.

\begin{definition}[$\alpha$-center proximity \cite{awasthi2012center}]
\label{defn:alphacp}
Let $(\mc X, d)$ be a clustering instance in some metric space $M$, and let $k$ be the number of clusters. We say that a center-based clustering $\mc C_{\mc X} = \{C_1, \ldots, C_k\}$ induced by centers $c_1, \ldots, c_k \in M$ satisfies the $\alpha$-center proximity property (with respect to  $\mc X$ and $k$) if the following holds 
$$\forall x \in C_i, i\neq j, \alpha d(x, c_i) < d(x, c_j)$$
\end{definition}

 This property has been considered in the past in various studies \cite{balcan2012clustering,awasthi2012center}. In this appendix we will show some connections between $\gamma$-margin and $\alpha$-center proximity properties.

  It is important to note that throughout this paper we considered clustering in Euclidean spaces. Furthermore, the centers were not restricted to be selected from the data points. 
  However, this is not necessarily the case in other studies.

\begin{table}[]
\centering
\caption{Known results for $\alpha$-center proximity}
\label{table:alphacp}
\begin{tabular}{lll}
\cline{2-3}
\multicolumn{1}{l|}{} & \multicolumn{1}{l|}{Euclidean} & \multicolumn{1}{l|}{General Metric} \\ \hline
\multicolumn{1}{|l|}{\begin{tabular}[c]{@{}l@{}}Centers \\ from data\end{tabular}} & \multicolumn{1}{l|}{\begin{tabular}[c]{@{}l@{}}Upper bound : $\sqrt{2}+1$  \cite{balcan2012clustering}\\ Lower bound : ?\end{tabular}} & \multicolumn{1}{l|}{\begin{tabular}[c]{@{}l@{}}Upper bound : $\sqrt{2}+1$  \cite{balcan2012clustering}\\ Lower bound : 2 \cite{ben2014data}\end{tabular}} \\ \hline
\multicolumn{1}{|l|}{\begin{tabular}[c]{@{}l@{}}Unrestricted \\ Centers \end{tabular}} & \multicolumn{1}{l|}{\begin{tabular}[c]{@{}l@{}}Upper bound : $2+\sqrt{3}$ \cite{awasthi2012center}\\ Lower bound : ?\end{tabular}} & \multicolumn{1}{l|}{\begin{tabular}[c]{@{}l@{}}Upper bound : $2+\sqrt{3}$ \cite{awasthi2012center}\\ Lower bound : 3 \cite{awasthi2012center}\end{tabular}} \\ \hline
 &  & 
\label{table:alphacenter}
\end{tabular}
\end{table}

An overview of the known results under $\alpha$-center proximity is provided in Table \ref{table:alphacenter}. The results are provided for the case that the centers are restricted to be selected from the training set, and also the unrestricted case (where the centers can be arbitrary points from the metric space). Note that any upper bound that works for general metric spaces also works for the Euclidean space. 

We will show that using the same techniques one can prove upper and lower bounds for $\gamma$-margin property. It is important to note that for $\gamma$-margin property, in some cases the upper and lower bounds match. Hence, there is no hope to further improve those bounds unless P=NP. A summary of our results is provided in \ref{table:gammamargin}.  

\begin{table}[]
\centering
\caption{Results for $\gamma$-margin}
\label{table:gammamargin}
\begin{tabular}{lll}
\cline{2-3}
\multicolumn{1}{l|}{}                                                                     & \multicolumn{1}{l|}{Euclidean} & \multicolumn{1}{l|}{General Metric}                                                                         \\ \hline
\multicolumn{1}{|l|}{\begin{tabular}[c]{@{}l@{}}Centers \\ from data\end{tabular}}        & \multicolumn{1}{l|}{\begin{tabular}[c]{@{}l@{}}Upper bound : 2 (Thm. \ref{thm:upperCenterData})\\ Lower bound : ? \end{tabular}}    &       \multicolumn{1}{l|}{\begin{tabular}[c]{@{}l@{}}Upper bound : 2 (Thm. \ref{thm:upperCenterData})\\ Lower bound : 2 (Thm. \ref{thm:lowerCenterData})\end{tabular}}           \\ \hline
\multicolumn{1}{|l|}{\begin{tabular}[c]{@{}l@{}}Unrestricted \\Centers \end{tabular}} & \multicolumn{1}{l|}{\begin{tabular}[c]{@{}l@{}}Upper bound : 3 (Thm. \ref{thm:upperCenterMetric})\\ Lower bound : 1.84 (Thm. \ref{thm:gammaLower})\\ \end{tabular}}         & \multicolumn{1}{l|}{\begin{tabular}[c]{@{}l@{}}Upper bound : 3 (Thm. \ref{thm:upperCenterMetric})\\ Lower bound : 3 (Thm. \ref{thm:lowerCenterMetric})\\ Awasthi\end{tabular}} \\ \hline
                                                                                          &                       &    
\label{table:gammamargin}                                                                                                                                                                                                 
\end{tabular}
\end{table}

\subsection{Centers from data}
\begin{theorem}
\label{thm:upperCenterData}
Let $(X , d)$ be a clustering instance and $\gamma \ge 2$. Then, Algorithm 1 in \cite{balcan2012clustering} outputs a tree $\mc T$ with the following property:

Any $k$-clustering $\mc C^* = \{C_1^*, \ldots, C_k^* \}$ which satisfies the $\gamma$-margin property and its cluster centers $\mu_1, \ldots, \mu_k$ are in $X$, is a pruning of the tree $T$. In other words, for every $1 \le i \le k$, there exists a node $N_i$ in the tree $T$ such that $C_i^* = N_i$.
\end{theorem}

\begin{proof}
Let $p, p' \in C_i^*$ and $q \in C_j^*$. \cite{balcan2012clustering} prove the correctness of their algorithm for $\alpha > \sqrt{2} + 1$. Their proof relies only on the following three properties which are implied when $\alpha > \sqrt{2} + 1$. We will show that these properties are implied by $\gamma > 2$ instances as well.
\begin{itemize}[nolistsep,noitemsep]
\item $d(p, \mu_i) < d(p, q)$\\
$\gamma d(p, \mu_i) < d(q, \mu_i) < d(p, q) + d(p, \mu_i) \implies d(p, \mu_i) < \frac{1}{\gamma-1}d(p, q)$.
\item $d(p, \mu_i) < d(q, \mu_i)$\\
This is trivially true since $\gamma > 2$.
\item $d(p, \mu_i) < d(p', q)$\\
Let $r = \max_{x \in C_i^*} d(x, \mu_i)$. Observe that $d(p, \mu_i) < r$. Also, $d(p', q)> d(q, \mu_i)-d(p', \mu_i) > \gamma r - r = (\gamma -1)r$.
\end{itemize}
\end{proof}

\begin{theorem}
\label{thm:lowerCenterData}
Let $(\mc X, d)$ be a clustering instance and $k$ be the number of clusters. For $\gamma < 2$, finding a $k$-clustering of $X$ which satisfies the $\gamma$-margin property and where the corresponding centers $\mu_1, \ldots, \mu_k$ belong to $\mc X$ is NP-Hard.
\end{theorem}
\begin{proof}
For $\alpha < 2$, \cite{ben2014data} proved that in general metric spaces, finding a clustering which satisfies the $\alpha$-center proximity and where the centers $\mu_1, \ldots, \mu_k \in \mc X$ is NP-Hard. Note that the reduced instance in their proof, also satisfies $\gamma$-margin for $\gamma < 2$. 
\end{proof}

\subsection{Centers from metric space}
\begin{theorem}
\label{thm:upperCenterMetric}
Let $(X , d)$ be a clustering instance and $\gamma \ge 3$. Then, the standard single-linkage algorithm outputs a tree $\mc T$ with the following property:

Any $k$-clustering $\mc C^* = \{C_1^*, \ldots, C_k^* \}$ which satisfies the $\gamma$-margin property is a pruning of $T$. In other words, for every $1 \le i \le k$, there exists a node $N_i$ in the tree $T$ such that $C_i^* = N_i$. 
\end{theorem}

\begin{proof}
\cite{balcan2008discriminative} showed that if a clustering $C^*$ has the strong stability property, then single-linkage outputs a tree with the required property. It is simple to see that if $\gamma > 3$ then instances have strong-stability and the claim follows.  
\end{proof}

\begin{theorem}
\label{thm:lowerCenterMetric}
Let $(\mc X, d)$ be a clustering instance and $\gamma < 3$. Then, finding a $k$-clustering of $X$ which satisfies the $\gamma$-margin is NP-Hard.
\end{theorem}
\begin{proof}
\cite{awasthi2012center} proved the above claim but for $\alpha < 3$ instances. Note however that the construction in their proof satisfies $\gamma$-margin for $\gamma < 3$. 
\end{proof}


\section{Proofs of Lemmas \ref{lemma:kmeansEquivalenceX3C} and \ref{lemma:gammaLower}}
\label{appendix:lowerBoundProof}

In Section \ref{section:lowerBounds} we proved Theorem \ref{thm:gammaLower} based on two technical results (i.e., lemma \ref{lemma:kmeansEquivalenceX3C} and \ref{lemma:gammaLower}). In this appendix we provide the proofs for these lemmas. In order to start, we first need to establish some properties about the Euclidean embedding of $X$ proposed in Section \ref{section:lowerBounds}.


\begin{definition}[$A$- and $B$-Clustering of $R_i$]
\label{defn:abclusteringVattani}

An $A$-Clustering of row $R_i$ is a clustering in the form of $\{\{s_i\}, \{r_{i,1}, r_{i,2}\}, \{r_{i,3}, r_{i,4}\}, \ldots,$ $ \{r_{i,6m-1}, r_{i,6m}\},\{r_{i, 6m+1}, f_i\}\}$. A $B$-Clustering of row $R_i$ is a clustering in the form of $\{\{s_i, r_{i, 1}\}, \{r_{i,2}, r_{i,3}\}, \{r_{i,4}, r_{i,5}\}, \ldots,$ $ \{r_{i,6m}, r_{i,6m+1}\},\{f_i\}\}$. 
\end{definition}

\begin{definition}[Good point for a cluster]
\label{defn:goodPointVattani}
A cluster $C$ is good for a point $z \not\in C$ if adding $z$ to $C$ increases cost by exactly $\frac{2w}{3}h^2$ 
\end{definition}

Given the above definition, the following simple observations can be made. 
\begin{itemize}[nolistsep,noitemsep]
\item The clusters $\{r_{i,2j-1}, r_{i, 2j}\}$, $\{r_{i,2j}, r_{i, 2j+1}\}$ and $\{g_{i,j}\}$ are good for $x_{i,j}$ and $y_{i-1,j}$.
\item The clusters $\{r_{i,2j}, r_{i, 2j+1}\}$ and $\{g_{i,j}\}$ are good for $x_{i,j}'$ and $y_{i-1,j}'$.
\end{itemize}

\begin{definition}[Nice Clustering]
\label{defn:niceClustering}
A $k$-clusteirng is nice if every $g_{i,j}$ is a singleton cluster, each $R_i$ is grouped in the form of either an $A$-clustering or a $B$-clustering, and each point in $Z_i$ is added to a cluster which is good for it.
\end{definition}

It is straightforward to see that a row grouped in a $A$-clustering costs $(6m+3)w-\alpha$ while a row in $B$-clustering costs $(6m+3)w$. Hence, a nice clustering of $H_{l,m} \cup Z$ costs at most $L_1 + L_2$. More specifically, if $t$ rows are group in a $A$-clustering, the nice-clustering costs $L_1+L_2-t\alpha$. Also, observe that any nice clustering of $X$ has only the following four different types of clusters. \begin{enumerate}[label=(\arabic*),nolistsep,leftmargin=*]
\item Type E - $\{r_{i,2j-1}, r_{i,2j+1}\}$ \\
The cost of this cluster is $2w$ and the contribution of each location to the cost (i.e., $\frac{cost}{\#locations}$) is $ \frac{2w}{2} = w$.
\item Type F - $\{r_{i,2j-1}, r_{i, 2j}, x_{i, j}\}$ or $\{r_{i,2j-1}, r_{i, 2j}, y_{i-1, j}\}$ or $\{r_{i,2j}, r_{i, 2j+1}, x_{i, j}'\}$ or $\{r_{i,2j}, r_{i, 2j+1}, y_{i-1, j}'\}$\\
The cost of any cluster of this type is $2w(1+\frac{h^2}{3})$ and the contribution of each location to the cost is at most $\frac{2w}{9}(h^2+3)$. This is equal to $\frac{16}{9}w$ because we had set $h = \sqrt 5$.
\item Type I - $\{g_{i, j}, x_{i,j}\}$ or $\{g_{i, j}, x_{i,j}'\}$  or $\{g_{i, j}, y_{i,j}\}$  or $\{g_{i, j}, y_{i,j}'\}$\\
The cost of any cluster of this type is $\frac{2}{3}wh^2$ and the contribution to the cost of each location is $\frac{w}{3}h^2$. For our choice of $h$, the contribution is $\frac{5}{3}w$.
\item Type J - $\{s_i, r_{i,1}\}$ or $\{r_{i,6m+1}, f_i\}$\\
The cost of this cluster is $3w$ (or $3w-\alpha$) and the contribution of each location to the cost is at most $1.5w$. 
\end{enumerate}
Hence, observe that in a nice-clustering, any location contributes at most $\le \frac{16}{9}w$ to the total clustering cost. This observation will be useful in the proof of the lemma below.

\begin{lemma}
\label{lemma:costNonNice}
For large enough $w = poly(l, m)$, any non-nice clustering of $X = H_{l, m} \cup Z$ costs at least $L + \frac{w}{3}$.
\end{lemma}

\begin{proof}
We will show that any non-nice clustering $C$ of $X$ costs at least $\frac{w}{3}$ more than any nice clustering. This will prove our result. The following cases are possible.

\begin{itemize}[nolistsep,leftmargin=*]
\item $C$ contains a cluster $C_i$ of cardinality $t > 6$ (i.e., contains $t$ weighted points)\\
Observe that any $x \in C_i$ has at least $t-5$ locations at a distance greater than 4 to it, and $4$ locations at a distance at least $2$ to it. Hence, the cost of $C_i$ is at least $\frac{w}{2t}(4^2(t-5)+2^24)t = 8w(t-4)$. $C_i$ allows us to use at most $t-2$ singletons. This is because a nice clustering of these $t+(t-2)$ points uses at most $t-1$ clusters and the clustering $C$ uses  $1 + (t-2)$ clusters for these points. The cost of the nice cluster on these points is $\le \frac{16w}{9}2(t-1)$. While the non-nice clustering costs at least $8w(t-4)$. For $t \ge 6.4 \implies 8(t-4) > \frac{32}{9}(t-1)$ and the claim follows. Note that in this case the difference in cost is at least $\frac{8w}{3}$. 

\item Contains a cluster of cardinality $t = 6$\\
Simple arguments show that amongst all clusters of cardinality $6$, the following has the minimum cost. $C_i = \{r_{i, 2j-1}, r_{i, 2j}, x_{i,j}, y_{i-1, j}, r_{i, 2j+1}, r_{2j+2}\}$. The cost of this cluster is $\frac{176w}{6}$. Arguing as before, this allows us to use $4$ singletons. Hence, a nice cluster on these $10$ points costs at most $\frac{160w}{9}$. The difference of cost is at least $34w$.  

\item Contains a cluster of cardinality $t = 5$\\
Simple arguments show that amongst all clusters of cardinality $5$, the following has the minimum cost. $C_i = \{r_{i, 2j-1}, r_{i, 2j}, x_{i,j}, y_{i-1, j}, r_{i, 2j+1}\}$. The cost of this cluster is $16w$. Arguing as before, this allows us to use $3$ singletons. Hence, a nice cluster on these $8$ points costs at most $16w\frac{8}{9}$. The difference of cost is at least $\frac{16w}{9}$.  

\item Contains a cluster of cardinality $t = 4$\\
It is easy to see that amongst all clusters of cardinality $4$, the following has the minimum cost. $C_i = \{r_{i, 2j-1}, r_{i, 2j}, x_{i,j}, r_{i, 2j+1}\}$. The cost of this cluster is $11w$. Arguing as before, this allows us to use $2$ singletons. Hence, a nice cluster on these $6$ points costs at most $\frac{32w}{3}$. The difference of cost is at least $\frac{w}{3}$.

\item All the clusters have cardinality $\le 3$ \\
Observe that amongst all non-nice clusters of cardinality $3$, the following has the minimum cost. $C_i = \{r_{i, 2j-1}, r_{i, 2j}, r_{i, 2j+1}\}$. The cost of this cluster is $8w$. Arguing as before, this allows us to use at most $1$ more singleton. Hence, a nice cluster on these $4$ points costs at most $\frac{64w}{9}$. The difference of cost is at least $\frac{8w}{9}$.

It is also simple to see that any non-nice clustering of size $2$ causes an increase in cost of at least $w$.

\end{itemize}
\end{proof}

\begin{proof}[Proof of lemma \ref{lemma:kmeansEquivalenceX3C}]
The proof is identical to the proof of Lemma 11 in \cite{vattani2009hardness}. Note that the parameters that we use are different with those utilized by \cite{vattani2009hardness}; however, this is not an issue, because we can invoke our lemma \ref{lemma:costNonNice} instead of the analogous result in Vattani (i.e., lemma 10 in Vattani's paper). The sketch of the proof is that based on lemma \ref{lemma:costNonNice}, only nice clusterings of $X$ cost $\le L$. On the other hand, a nice clustering corresponds to an exact 3-set cover. Therefore, if there exists a clustering of $X$ of cost $\le L$, then there is an exact 3-set cover. The other way is simpler to proof; assume that there exists an exact 3-set cover. Then, the corresponding construction of $X$ makes sure that it will be clustered \emph{nicely}, and therefore will cost $\le L$.

\end{proof}

\begin{proof}[Proof of lemma \ref{lemma:gammaLower}]
As argued before, any nice clustering has four different types of clusters. We will calculate the minimum ratio $a_i = \frac{d(y, \mu)}{d(x, \mu)}$ for each of these clusters $C_i$ (where $x \in C_i$, $y \not\in C_i$ and $\mu$ is mean of all the points in $C_i$.) Then, the minimum $a_i$ will give the desired $\gamma$. 
\begin{enumerate}[label=(\arabic*),nolistsep,leftmargin=*]
\item For Type E clusters $a_i = h/1 = \sqrt{5}$. 
\item For Type F clusters. $a_i = \frac{\frac{\sqrt{4+16(h^2-1)}}{3}}{2h/3} = \sqrt{\frac{17}{5}} \approx 1.84$. 
\item For Type I clusters, standard calculation show that $a_i > 2$.
\item For Type J clusters $a_i = \frac{2+\frac{\sqrt{6}}{2}}{\frac{\sqrt{6}}{2}} > 2$.
\end{enumerate}

\noindent Furthurmore, $|\mc X| = (12lm + 3l -6m)w$ and $k = 6lm + 2l - 3m$. Hence for $w = $poly$(l, m)$ our hardness result holds for $k = |\mc X|^{\epsilon}$ for any $0 < \epsilon < 1$.
\end{proof}
\noindent Lemmas \ref{lemma:kmeansEquivalenceX3C} and \ref{lemma:gammaLower} complete the proof of the main result (Thm. \ref{thm:gammaLower}).


\section{Concentration inequalities}
\label{appendixsection:conIneq}

\begin{theorem}[Generalized Hoeffding's Inequality (e.g., \cite{ashtiani2015dimension})]
\label{thm:genHoeff}
Let $X_1, \ldots. X_n$ be i.i.d random vectors in some Hilbert space such that for all $i$, $\|X_i\|_2 \le R$ and $E[X_i] = \mu$. If $n > c\frac{\log(1/\delta)}{\epsilon^2}$, then with probability atleast $1-\delta$, we have that
$$\Big\|\mu - \frac{1}{n}\sum X_i\Big\|_2^2 \le R^2\epsilon$$ 
\end{theorem}

\end{document}